\documentclass[final,12pt]{colt2019} 

\usepackage[english]{babel}

\usepackage[colorinlistoftodos]{todonotes}
\usepackage{wrapfig}
\usepackage{amsfonts}

\usepackage{algorithm}
\usepackage{algorithmic}
\usepackage[utf8]{inputenc} 
\usepackage[T1]{fontenc}    
\usepackage{hyperref}       
\usepackage{url}            
\usepackage{booktabs}       
\usepackage{amsfonts}       
\usepackage{nicefrac}       
\usepackage{microtype}      
\usepackage{mathtools}
\usepackage{thmtools}
\usepackage{thm-restate}
\usepackage[title]{appendix}
\usepackage{comment}
\usepackage{dsfont}

\usepackage{times}

\newtheorem{theorem-rst}[theorem]{Theorem}
\newtheorem{lemma-rst}[theorem]{Lemma}
\newtheorem{proposition-rst}[theorem]{Proposition}
\newtheorem{assumption-rst}[theorem]{Assumption}
\newtheorem{definition-rst}[theorem]{Definition}
\newtheorem{claim-rst}[theorem]{Claim}
\newtheorem{corollary-rst}[theorem]{Corollary}

\DeclarePairedDelimiter\br{(}{)}
\DeclarePairedDelimiter\brs{[}{]}
\DeclarePairedDelimiter\brc{\{}{\}}

\DeclarePairedDelimiter\abs{\lvert}{\rvert}
\DeclarePairedDelimiter\norm{\lVert}{\rVert}

\DeclarePairedDelimiter\ceil{\lceil}{\rceil}

\newcommand{\E}{\mathbb{E}}
\newcommand{\N}{\mathbb{N}}
\newcommand{\R}{\mathbb{R}}
\newcommand{\HE}{\mathbb{H}} 
\newcommand{\G}{\mathbb{G}} 
\newcommand{\F}{\mathcal{F}} 

\newcommand{\Var}{\mathrm{Var}} 

\newcommand{\oracleApp}{\alpha } 
\newcommand{\oracleProb}{\beta} 

\newcommand{\appFac}{\oracleApp\oracleProb} 

\newcommand{\Nbatch}{K}
\newcommand{\Narms}{L}
\newcommand{\Nitems}{M}

\newcommand{\dr}[1]{\Delta_{#1}}
\newcommand{\Ind}[1]{\mathds{1}\brc*{#1}}

\newcommand{\kl}{D_\mathrm{KL}}
\newcommand{\klBin}{\mathrm{kl}}

\title[Batch-Size Independent Regret Bounds for the CMAB Problem]{Batch-Size Independent Regret Bounds for the Combinatorial Multi-Armed Bandit Problem}

\coltauthor{%
 \Name{Nadav Merlis} \Email{merlis@campus.technion.ac.il}\\
 \addr Faculty of Electrical Engineering, Technion, Israel Institute of Technology.
 \AND
 \Name{Shie Mannor} \Email{shie@ee.technion.ac.il}\\
 \addr Faculty of Electrical Engineering, Technion, Israel Institute of Technology.
}

\begin{document}

\maketitle

\begin{abstract}%

We consider the combinatorial multi-armed bandit (CMAB) problem, where the reward function is nonlinear. In this setting, the agent chooses a batch of arms on each round and receives feedback from each arm of the batch. The reward that the agent aims to maximize is a function of the selected arms and their expectations. In many applications, the reward function is highly nonlinear, and the performance of existing algorithms relies on a global Lipschitz constant to encapsulate the function's nonlinearity. This may lead to loose regret bounds, since by itself, a large gradient does not necessarily cause a large regret, but only in regions where the uncertainty in the reward's parameters is high. 
To overcome this problem, we introduce a new smoothness criterion, which we term \emph{Gini-weighted smoothness}, that takes into account both the nonlinearity of the reward and concentration properties of the arms. We show that a linear dependence of the regret in the batch size in existing algorithms can be replaced by this smoothness parameter. This, in turn, leads to much tighter regret bounds when the smoothness parameter is batch-size independent.
For example, in the probabilistic maximum coverage (PMC) problem, that has many applications, including influence maximization, diverse recommendations and more, we achieve dramatic improvements in the upper bounds. We also prove matching lower bounds for the PMC problem and show that our algorithm is tight, up to a logarithmic factor in the problem's parameters.

\end{abstract}


\section{Introduction}


The multi-armed bandit (MAB) problem is one of the most elementary problems in decision making under uncertainty. Under this setting, the agent must choose an action (or arm) on each round, out of $\Narms$ possible actions. It then observes the chosen arm's reward, which is generated from some fixed unknown distribution, and aims to maximize the average cumulative reward \citep{robbins1952some}. Equivalently, the agent minimizes its expected cumulative regret, i.e., the difference between the best achievable reward and the agent's cumulative reward. This framework enables us to understand and control the trade-off between information gathering (`exploration') and reward maximization (`exploitation'), and many of the current reinforcement learning algorithms are based on exploration concepts that originate from MAB \citep{jaksch2010near,bellemare2016unifying,osband2013more,gopalan2015thompson}.

One of the active research directions in MAB consists in extending the model to support more complicated feedbacks from the environment and more complex reward functions. An important extension that follows this direction is the combinatorial multi-armed-bandit (CMAB) problem with semi-bandit feedback \citep{chen2016combinatorialA}. Instead of choosing a single arm, the agent selects a subset of the arms $A_t$ (a `batch'), and observes feedback $X_a(t)$ from each of the arms $a\in A_t$ (`semi bandit feedback'). The reward can be a general function of the expectations $\mu_a=\E\brs*{X_a(t)}$, with the linear function as the most common example: $r(A_t;\mu)=\sum_{a\in A_t} \mu_a$ for any batch of size $\abs{A_t}\le \Nbatch$.

Another common case that falls under the CMAB framework, and will benefit from the results of this paper, is the bandit version of the probabilistic maximum coverage (PMC) problem. Under this setting, each arm is a random set that may contain some subset of $\Nitems$ possible items, according to a fixed probability distribution. On each round, the agent chooses a batch of $\Nbatch$ sets and aims to maximize the number of items that appear in any of the sets (i.e., the size of the union of the sets). In the bandit setting, we assume that the probabilities that items appear in a set are unknown and aim to maximize the item coverage while concurrently learning the probabilities. Variants of the PMC bandit problem have many practical applications, such as influence maximization \citep{vaswani2015influence}, ranked recommendations \citep{kveton2015cascading}, wireless channel monitoring \citep{arora2011sequential}, online advertisement placement \citep{chen2016combinatorialA} and more. 

Although existing algorithms offer solutions to the CMAB framework under very general assumptions \citep{chen2016combinatorialA,chen2016combinatorialB}, there are several issues that still pose a major challenge in the design and analysis of practical algorithms. Notably, most of the existing algorithms quantify the nonlinearity of the function using a global Lipschitz constant. However, large gradients do not necessarily translate to a large regret, but rather the combined influence of the gradient size and the local parameter uncertainty.
More specifically, if there are regions in which tight parameter estimates can be derived, the regret will not be large even if the gradients are large. Conversely, in regions where the parameter uncertainty is large, smaller gradients can still cause a large regret.
For example, consider the reward function $r(A;\mu)=\sum_{a\in A}\min\brc*{c\mu_a,1}$, for $c\gg1$ and $\mu_a\in\brs*{0,1}$, and its translated version $r'(A;\mu)=\sum_{a\in A}\max\brc*{\min\brc*{c(\mu_a-\frac{1}{2},1},0}$. Despite the fact that both functions have the same Lipschitz constant, the regret in the first problem can be much smaller, since parameters are easier to estimate when they are close to the edge of their domain. Another similar example is the PMC bandit problem, in which the reward is approximately linear when the coverage probabilities are small, but declines exponentially when they are large. Thus, the na\"ive bound cannot capture the interaction between the gradient size and the parameter uncertainty, which results in loose regret bounds.

In this paper, we aim to utilize this principle for the CMAB framework. To this end, we introduce a new smoothness criterion called \textit{Gini-weighted smoothness}. This criterion takes into account both the nonlinearity of the reward function and the concentration properties of bounded random variables around the edges of their domain. We then suggest an upper confidence bound (UCB) based strategy \citep{auer2002finite}, but replace the classical Hoeffding-type confidence intervals with ones that depend on the empirical variance of the arms, and are based on the Empirical Bernstein inequality \citep{audibert2009exploration}. 
We show that Bernstein-type bounds capture similar properties to the Gini-weighted smoothness, and thus allow us to derive tighter regret bounds. Notably, the linear dependence of the regret bound in the batch size is almost completely removed, except for a logarithmic factor, and the batch size only affects the regret through the Gini-smoothness parameter. In problems in which this parameter is batch-size independent, including the PMC bandit problem, our new bound is tighter by a factor of the batch size $\Nbatch$. This is comparable to the best possible improvement due to independence assumption in the linear CMAB problem \citep{degenne2016combinatorial}, but without any additional statistical assumptions. 

Moreover, we demonstrate the tightness of our regret bounds by proving matching lower bounds for the PMC bandit problem, up to logarithmic factors in the batch size. To do so, we construct an instance of the PMC bandit problem that is equivalent to a classical MAB problem and then analyze the lower bounds of this problem. We also show that in contrast to the linear CMAB problem, the lower bounds do not change even if different sets are independent. To the best of our knowledge, our algorithm is the first to achieve tight regret bounds for the PMC bandit problem.

\section{Related Work}

The multi-armed bandits' literature is vast. We thus cover only some aspects of the area, and refer the reader to \citep{bubeck2012regret} and \citep{lattimore2018bandit} for a comprehensive survey. We employ the Optimism in the Face of Uncertainty (OFU) principle \citep{lai1985asymptotically}, which is one of the most fundamental concepts in MAB, and can be found in many known MAB algorithms (e.g., \citealt{auer2002finite,garivier2011kl}). While many algorithms rely on Hoeffding-type concentration bounds to derive an upper confidence bound (UCB) of an arm, a few previous works also apply Bernstein-type bounds and demonstrate superior performance, both in theory and in practice \citep{audibert2009exploration,mukherjee2018efficient}. 

The general stochastic combinatorial multi-armed bandit framework was introduced in \citep{chen2013combinatorial}. They presented CUCB, an algorithm that uses UCB per arm (or 'base arm'), and then inputs the optimistic value of the arm into a maximization oracle. 
Many preceding works also fall under the CMAB framework, but mainly focus on a specific reward function \citep{caro2007dynamic,gai2010learning,gai2012combinatorial,liu2012adaptive}, or work in the adversarial setting \citep{cesa2012combinatorial}.  While most algorithms for the CMAB setting follow the OFU principle, a few employ Thompson Sampling \citep{thompson1933likelihood}, e.g., \citep{wang2018thompson,huyuk2018thompson} for the semi-bandit feedback and \citep{gopalan2014thompson} for the full-bandit feedback. In recent years, there have been extensive studies on deriving tighter bounds, but these works mostly address the linear CMAB problem \citep{kveton2014matroid,kveton2015tight,combes2015combinatorial,degenne2016combinatorial}. More recently, tighter regret bounds were derived for this framework in \citep{wang2017improving}, and also allow probabilistically triggered arms. Nevertheless, we show that in our setting, these bounds can be improved by a factor of the batch size $\Nbatch$ for many problems, e.g., the PMC bandit problem, and are comparable, up to logarithmic factors, otherwise. 

Empirical Bernstein Inequality was first used for the CMAB problem in \citep{gisselbrecht2015whichstreams} for linear reward functions, and was later used in \citep{perrault2018finding} for sequential search-and-stop problems. Both works focus on specific reward functions and utilize Bernstein inequality to get variance-dependent regret bounds. In contrast, we analyze general reward functions and exploit the relation between the confidence interval and the reward function to derive tighter regret bounds. 

The PMC bandit problem is the bandit version of the maximum coverage problem \citep{hochbaum1996approximation},
a well studied subject in computer science with many variants and extensions. 
The bandit variant is closely related to the influence maximization problem \citep{vaswani2015influence,carpentier2016revealing,wen2017online}, in which the agent chooses a set of nodes in a graph, that influence other nodes through random edges, and aims to maximize the number of influenced nodes. 
Another related setting is the cascading bandit problem (e.g., \citealt{kveton2015cascading,kveton2015combinatorial,combes2015learning,lattimore2018toprank}), in which a list of items is sequentially shown to a user until she finds one of them satisfactory. This is equivalent to a coverage problem with a single object, but only partial feedback - the user will not give any feedback about items that appear after the one she liked. In both settings, the focus is very different than ours. In influence maximization, the focus is on the diffusion inside a graph and the graph structure, and in cascading bandits on the partial feedback and the list ordering. They are thus complementary to our framework and could benefit from our results.

\section{Preliminaries} \label{section:prelim}

We work under the stochastic combinatorial-semi bandits framework, when the reward is the weighted sum of smooth monotonic functions.  Assume that there are $\Narms$ arms (`base arms'), and let $\mathcal{A}\subset 2^{\brs*{\Narms}}$ be the action set, i.e., the collection of possible batches (actions) from which the agent can choose, with $\brs*{\Narms}=\brc*{1,
\dots,\Narms}$. Also assume that the size of any batch $A\in\mathcal{A}$ is bounded by $\abs{A}\le\Nbatch$. Denote the reward function of an action $A$ with arm parameters $p$ by $r(A;p)$ and assume that the reward is the weighted sum of $\Nitems$ functions, $r(A;p)=\sum_{i=1}^\Nitems w_i r_i(A;p_i)$, for some fixed weights $\brc*{{w_i}}_{i=1}^\Nitems, w_i\ge0$ and $p_i\in\R^L$. Without loss of generality, we also assume $r_i(A;x)\ge0$. 

The agent interacts with the environment as follows: on each round $t$, the agent chooses an action $A_t\in\mathcal{A}$. 
Then, for each arm $j\in A_t$, it observes feedback $X_{ij}(t)\in\brs*{0,1}$ for any $i\in\brs*{\Nitems}$, with mean $\E\brs*{X_{ij}(t)}=p_{ij}$. For ease of notation, assume that $X_{ij}(t)=0\,$ if $\,j\notin A_t$. 
We denote the empirical estimators of the parameters $p_{ij}$ by 
$\hat{p}_{ij}(t)=\frac{1}{N_j(t)}\sum_{\tau=1}^t X_{ij}(\tau)$, where $N_j\br*{t}=\sum_{\tau=1}^t \Ind{j\in A_\tau}$ is the number of times an arm $j$ was chosen up to time $t$, and $\Ind{\cdot}$ is the indicator function. We also denote the empirical variance by $\hat{V}_{ij}(t)=\frac{1}{N_j(t)}\sum_{\tau=1}^t X_{ij}^2(\tau) - \br*{\hat{p}_{ij}(t)}^2$. The estimated parameters $\hat{p}_{ij}(t)$ are concentrated around their mean according to the Empirical Bernstein inequality:

\begin{lemma}[Empirical Bernstein]\citep{audibert2009exploration}
\label{lemma:emp_bernstein}
Let $X_1,\dots,X_s$ be i.i.d random variables taking their values in $[0,1]$, and let $\mu=\E \brs*{X_i}$ be their common expected value. Consider the empirical mean $\bar{X}_s$ and variance $V_s$ defined respectively by
\begin{center}
$\bar{X}_s=\frac{1}{s}\sum_{i=1}^s X_i \enspace$ and $\enspace V_s=\frac{1}{s}\sum_{i=1}^s \br*{X_i-\bar{X}_s}^2=\frac{1}{s}\br*{\sum_{i=1}^s X_s^2} - \bar{X}_s^2$.
\end{center}
Then, for any $s\in\N$ and $x>0$, it holds that $\Pr\brc*{\abs{\bar{X}_s-\mu} \ge \sqrt{\frac{2V_s x}{s}} + \frac{3 x}{s}}\le 3e^{-x}$.
\end{lemma}

We require the functions $r_i(A;p_i)$ to be monotonic Gini-smooth with smoothness parameters $\gamma_\infty,\gamma_g$, which we define in the following:
\begin{definition}
\label{def:GiniSmooth}
Let $f(A;x):\mathcal{A}\times\brs*{0,1}^\Narms\to\R$ be a differentiable function in $x\in\br*{0,1}^\Narms$ and continuous in $x\in\brs*{0,1}^\Narms$, for any $A\in\mathcal{A}$. The function $f(A;x)$ is said to be monotonic Gini-smooth, with smoothness parameters $\gamma_\infty$ and $\gamma_g$, if:
\begin{enumerate}
    \item For any $A\in\mathcal{A}$, the function is monotonically increasing with bounded gradient, i.e., for any $i\in A$ and $x\in\br*{0,1}^\Narms$, $0\le \frac{\partial f(A;x)}{\partial x_i} \le \gamma_\infty$. If $i\notin A$, then $\frac{\partial f(A;x)}{\partial x_i}=0$ for all $x\in\br*{0,1}^\Narms$.
    \item For any $A\in\mathcal{A}$ and $x\in\br*{0,1}^\Narms$, it holds that \vspace {-0.1cm}
    {\small
    \begin{equation}  
    \label{eq:smoothDef}
    \sqrt{\sum_{i=1}^\Narms x_i(1-x_i) \br*{\frac{\partial f(A;x)}{\partial x_i}}^2}\le \gamma_g \enspace .
    \end{equation} }
    Throughout the paper, we refer this condition as the Gini-weighted smoothness\footnote{The name is motivated by the similarity of the weights to the Gini impurity $\sum_i x_i(1-x_i)$.} of $f(A;x)$.
\end{enumerate}
\end{definition}

While the first condition is very intuitive, and is equivalent to the standard smoothness requirement \citep{wang2017improving}, the second condition demands further explanation. 
Notably, the Gini-smoothness parameter is less sensitive to changes in the reward function when the parameters are close to the edges, i.e., close to $0$ or $1$. 
It may be observed that in these regions, the variances of the parameters are small, which implies that they are more concentrated around their mean. We will later show that this will allow us to mitigate the effect of large gradients on the regret. For simplicity, we assume that all of the functions $r_i(A;p_i)$ have the same smoothness parameters, but the extension is trivial. 
We note that if for all $i\in\brs*{\Nitems}$, the functions $r_i(A;p_i)$ are Gini-smooth, then $r(A;p)$ is also Gini-smooth. Nevertheless, explicitly decomposing the reward into a sum of monotonic Gini-smooth functions leads to a slightly tighter regret bound. We believe that this is due to a technical artefact, but nonetheless modify our analysis since many important cases fall under this scenario. 

An important example that falls under our model is the Probabilistic Maximum Coverage (PMC) problem. In this setting, each arm is a random set that may contain some subset of $\Nitems$ possible items. The agent's goal is to choose a batch of sets such that as many items appear in the sets (i.e., the union of the sets is maximized). Formally, the functions $r_i(A;p)=1-\prod_{j\in A} \br*{1-p_{ij}}$ are the probabilities that an item $i\in\brs*{\Nitems}$ was covered and $r(A;p)$ is the (weighted) expected number of covered items. The smoothness constants for this problem are $\gamma_\infty=1$ and $\gamma_g=\frac{1}{\sqrt{e}}$, independently of the batch size. Another example is the logistic function $r_i(A;p)=\frac{1}{1+C e^{-\sum_{j\in A} p_{ij}}}$, for which the smoothness parameters are equal to $\gamma_\infty=\frac{1}{4}$ and $\gamma_g=\frac{1}{4}\sqrt{1+\log C}$, which are batch size independent, provided that $C$ does not depend on $\Nbatch$. In the linear case, $r_i(A;p)=r(A;p)=\sum_{j\in A} p_{j}$ and the smoothness parameters are $\gamma_\infty=1$ and $\gamma_g=\sqrt{\Nbatch/4}$. In general, we can always bound $\gamma_g \le \frac{1}{2}\sqrt{\Nbatch}\gamma_\infty$, and we will later see that in this case, our bounds are comparable to existing results that only rely on $\gamma_\infty$. 

Similarly to previous work, the performance of an agent is measured according to its regret, i.e., the difference between the reward of an oracle, which acts optimally according to the statistics of the arms, and the cumulative reward of the agent.
However, in many problems, it is impractical to achieve the optimal solution even when the problem's parameters are known. Thus, it is more sensible to compare the performance of an algorithm to the best approximate solution that an oracle can achieve. Denote the optimal action by $A^*=\arg\max_{A\in\mathcal{A}}\brc*{r(A;p)}$ and its value by $r_{\max}$.
An oracle is called an $(\oracleApp,\oracleProb)$ \textit{approximation oracle} if for every parameter set $p$, with probability $\oracleProb$, it outputs a solution whose value is at least $\oracleApp r_{\max}$. We define the \textit{expected approximation regret} of an algorithm as:
\begin{equation} 
\label{eq:regretDef}
R(T) = \appFac r_{\max} T - \sum_{t=1}^T  \E\brs*{r\br*{A_t;p}} \enspace ,
\end{equation}
where the expectation is over the randomness of the environment and the agent's actions, through the oracle. As noted by \citep{chen2016combinatorialA}, in the linear case, and sometimes when arms are independent, the reward function equals to the expectation of the empirical reward, i.e., $r(A_t;\mu)=\E\brs*{r(A_t;X)}$ but unfortunately, it does not necessarily occur when arms are arbitrarily correlated.

We end this section with some notations. Let $\bar{\Nitems}=\sum_{i=1}^M w_i$. Denote the suboptimality gap of a batch $A_t$ by $\dr{A_t} = \oracleApp r_{\max} - r\br*{A_t;p}$. 
The minimal gap of a base arm $j$ is the smallest positive gap of a batch $A$ that contains arm $j$, namely,  $\dr{j,\min}=\min_{A\in\mathcal{A},j\in A,\dr{A}>0}\dr{A}$ and the maximal gap is $\dr{j,\max}=\max_{A\in\mathcal{A},j\in A,\dr{A}>0}\dr{A}$. Note that for all $A\in \mathcal{A}$, $\dr{A} \le \max_j \dr{j,\max}\triangleq \dr{\max}$.
We denote by $\kl(X,Y)$ the Kullback-Leibler (KL) divergence between two random variables $X,Y$. 
Also denote by $\klBin(p,q)=p\log\frac{p}{q} + (1-p)\log\frac{1-p}{1-q}$, the KL divergence between two Bernoulli random variables of means $p$ and $q$.

\section{Algorithm}

\begin{algorithm}[H]
\caption{BC-UCB (\textbf{B}ernstein \textbf{C}ombinatorial - \textbf{U}pper \textbf{C}onfidence \textbf{B}ound)}
\label{alg:BC-UCB}
\begin{algorithmic}[1]
    \STATE Initialize arm counts $N_j(0)=0$
    \FOR{$t=1,\dots,T$}
        \IF{$\exists j\in\brs*{\Narms}: N_j(t-1)=0$}
            \STATE Choose batch $A_t$ such that $j\in A_t$
            
        \ELSE
            \STATE $q_{ij}(t) = \min\brc*{1, \hat{p}_{ij}(t-1) + \sqrt{\frac{6 \hat{V}_{ij}(t-1)\log t}{N_{j}(t-1)}} + \frac{9\log t}{N_{j}(t-1)}}$
            \STATE $A_t=\mathrm{Oracle}(q)$
        \ENDIF
        \STATE Play $A_t$ and observe $X_{ij}(t)$ for all $j\in A_t$ and $i\in\brs*{\Nitems}$
        \STATE Update $N_j(t), \hat{p}_{ij}(t)$ and $\hat{V}_{ij}(t)$ for all $j\in A_t$ and  $i\in\brs*{\Nitems}$
    \ENDFOR
\end{algorithmic}
\end{algorithm}

We suggest a combinatorial UCB-type algorithm with Bernstein based confidence interval, which we call BC-UCB (\textbf{B}ernstein \textbf{C}ombinatorial - \textbf{U}pper \textbf{C}onfidence \textbf{B}ound). The UCB index is defined as 
\begin{equation}
    q_{ij}(t) = \min\brc*{1, \hat{p}_{ij}(t-1) + \sqrt{\frac{6 \hat{V}_{ij}(t-1)\log t}{N_{j}(t-1)}} + \frac{9\log t}{N_{j}(t-1)}} \enspace.
\end{equation}
A pseudocode can be found in Algorithm \ref{alg:BC-UCB}. On the first rounds, the agent initially samples batches to make sure that each base arm is sampled at least once. Afterwards, we calculate the UCB index $q$, and an approximation oracle chooses an action 
$A_t$ such that $r(A_t;q) \ge \oracleApp r(A^*(q);q)$ with probability $\oracleProb$. The agent then plays this action and observes feedback for any arm $j\in A_t$. It finally updates the empirical probabilities and variances and continues to the next round. 

An example for an approximation oracle that can be used when the reward function is monotonic and submodular \footnote{Let $\Omega$ be a finite set. A set function $f:2^\Omega\to\R$ is called submodular if $\forall S,T\!\in\!\Omega, f(S)+f(T) \!\ge\! f(S \cup T) + f(S \cap T)$.} is the greedy oracle, which enjoys an approximation factor of $\oracleApp=1-\frac{1}{e}$ with probability $\oracleProb=1$ \citep{nemhauser1978analysis}. First, the oracle initializes $A_t(0)=\emptyset$, and then selects $\Nbatch$ items sequentially in a greedy manner, i.e. $j_n\in\arg\max_{j\notin A_t(n-1)}\brc*{r\br*{A_t(n-1)\cup \brc*{j};q}}$ and $A_t(n)=A_t(n-1)\cup \brc*{j_n}$, with $A_t=A_t(\Nbatch)$. 

The first term of the confidence interval resembles the standard UCB term of $\sqrt{\frac{3\log t}{2 N_{j}(t-1)}}$, and is actually always smaller, since the empirical variance of variables in $[0,1]$ is always bounded by $\frac{1}{4}$.\footnote{$V=\frac{1}{n}\sum_i X_i^2 - \br*{\frac{1}{n}\sum_i X_i}^2\le \frac{1}{n}\sum_i X_i - \br*{\frac{1}{n}\sum_i X_i}^2 = \br*{\frac{1}{n}\sum_i X_i}\br*{1-\frac{1}{n}\sum_i X_i}\le \frac{1}{4}$.} The second term does not appear in the standard UCB bound, and can slightly affect the regret, since suboptimal arms will be asymptotically sampled $O\br*{\log t}$ times, so the two terms are comparable. Nevertheless, the first term is still dominant, and if the variance of an arm is drastically lower than $\frac{1}{4}$, the confidence bound is significantly tighter. This can happen, for example, when the arm's mean is close to $0$ or $1$. 

We take advantage of this property through the smoothness parameter $\gamma_g$, that only takes into account the sensitivity of the function to parameter changes when the parameters are far away from $0$ or $1$. This will allow us to derive a drastically tighter regret bound, in comparison to existing algorithms, when $\gamma_g = o(\sqrt{\Nbatch}\gamma_\infty)$, as we establish in the following theorem:

\begin{restatable}{theorem-rst}{dependentRegret}
\label{theorem:dependentRegret}
Let $r_i(A;p_i)$ be monotonic Gini-smooth reward functions with smoothness parameters $\gamma_\infty$ and $\gamma_g$, and let $r(A;p)=\sum_{i=1}^\Nitems w_i r_i(A;p_i)$ be the reward function.
For any $T\ge1$, the expected approximation regret of BC-UCB with ($\oracleApp$,$\oracleProb$)-approximation oracle is bounded by
\begin{align}
\label{eq:dependentRegret}
R(T) \le  &
\brs*{8640\gamma_g^2\bar\Nitems^2\sum_{j=1}^{\Narms}  \frac{1}{\dr{j,\min}}+340\gamma_\infty \bar\Nitems \sum_{j=1}^{\Narms}  \br*{1 + \log\frac{\dr{j,\max}}{\dr{j,\min}}}}\ceil*{\frac{\log \Nbatch}{1.61}}^2 \log T \nonumber\\
&+ \Narms\dr{\max} \br*{ 1 + \Nitems\frac{2\pi^2}{3}}.
\end{align}
\end{restatable}

We can also exploit the problem dependent regret bound to derive a problem independent bound, that is, a bound that holds for any gaps $\dr{j,\min}$:

\begin{restatable}{corollary-rst}{independentRegret}
\label{corollary:independentRegret}
Let $r_i(A;p_i)$ be monotonic Gini-smooth reward functions with smoothnsess parameters $\gamma_\infty$ and $\gamma_g$, and let $r(A;p)=\sum_{i=1}^\Nitems w_i r_i(A;p_i)$ be the reward function. 
For any $T\ge1$, the expected approximation regret of BC-UCB with ($\oracleApp$,$\oracleProb$)-approximation oracle can be bounded by
\begin{align}
R(T) & \le 
2\sqrt{8640}\gamma_g\bar\Nitems \ceil*{\frac{\log \Nbatch}{1.61}} \sqrt{\Narms T\log T} +  \Narms\dr{\max} \br*{ 1 + \Nitems\frac{2\pi^2}{3}} \nonumber\\
& \quad + 340\gamma_\infty \bar\Nitems\Narms\ceil*{\frac{\log \Nbatch}{1.61}}^2 \log T
\br*{2 + \log\frac{\dr{\max}T}{340\gamma_\infty \bar\Nitems \Narms\ceil*{\frac{\log \Nbatch}{1.61}}^2\log T}} \enspace .
\end{align}
\end{restatable}

The proof of Theorem \ref{theorem:dependentRegret} is presented in the following section, along with a proof sketch for Corollary \ref{corollary:independentRegret}. The full proof of the corollary can be found in Appendix \ref{append:independentRegret}.

We start by noting that we could avoid decomposing the reward into sum of $\Nitems$ functions, but the regret bound in this case is slightly looser - the $\bar\Nitems^2$ factor in (\ref{eq:dependentRegret}) is replaced by the larger factor $\Nitems\sum_i w_i^2$, and the $\log\Nbatch$ factor is replaced by $\log\Nbatch\Nitems$. We believe that the logarithmic factor $\log\Nbatch$ is due to a technical artefact, but leave its removal for future work. We also remark that the second term of the problem dependent regret bounds is negligible for small gaps and can always be bounded using the identity $1+\log x\le x$, which yields a regret of $O\br*{\br*{\bar\Nitems^2\gamma_g^2+\bar\Nitems\gamma_\infty\dr{\max}}\ceil*{\frac{\log \Nbatch}{1.61}}^2\sum_{j=1}^{\Narms}  \frac{\log T}{\dr{j,\min}}}$.

To the best of our knowledge, the closest bound to ours appears in \citep{wang2017improving}. From their perspective, our bandit problem has $\Nitems\Narms$ base arms and a batch size of $\Nitems\Nbatch$ with gaps $\dr{ij,\min}=\dr{j,\min}$. Their $L_1$ smoothness parameter 
equals $B=\gamma_\infty \max_i  w_i$. Substituting into their bound yields a regret of $O\br*{\sum_{i,j} \frac{\br*{\gamma_\infty\max_i  w_i}^2\Nitems\Nbatch\log T}{\dr{ij,\min}}}=O\br*{\gamma_\infty^2\Nbatch\sum_{j} \frac{\br*{\max_i  w_i}^2\Nitems^2\log T}{\dr{j,\min}}}$.
Since $\gamma_g\le \frac{1}{2}\sqrt{\Nbatch}\gamma_\infty$, $\dr{\max}\le \bar\Nitems\Nbatch\gamma_\infty$ and $\bar\Nitems\le \Nitems \max_i  w_i$, our regret is tighter when $\gamma_g$ is not trivial (that is, $\gamma_g=o(\sqrt{\Nbatch}\gamma_\infty)$), up to a factor of $\log^2 \Nbatch$.

Alternatively to our approach, it is possible to analyze BC-UCB only based on the $\gamma_\infty$ smoothness. In this case, the analysis will be very similar to that of \cite{kveton2015tight}. This will yield a dominant term that does not depend on the logarithmic factor $\log\Nbatch$, and declines with the variance of the arms. However, it can lead to dramatically worse bounds when $\gamma_g$ is small, so we decided not to pursue this path. Nonetheless, this approach is still worth mentioning when comparing to other algorithms, since when its bounds are combined with ours, we can conclude that the regret of BC-UCB is always tighter than the regret obtained in \citep{wang2017improving}. 
 
On a final note, we return to the PMC bandit problem. In this case, $\gamma_g$ and $\gamma_\infty$ are $O(1)$ and $\dr{\max}$ is $O(\bar\Nitems)$, so the regret is $O\br*{\bar\Nitems^2 \log^2\Nbatch \sum_{j=1}^{\Narms} \frac{\log T}{\dr{j,\min}}}$, which is tighter by a factor of $\Nbatch$ from existing results \citep{wang2017improving}. We will later show that this bound is tight, up to logarithmic factors in $\Nbatch$.

\section{Proving the Regret Upper Bounds}

We start the proof by simplifying the first term of the UCB index. To this end, recall Bernstein's inequality: 

\begin{lemma}[Bernstein's Inequality]
\label{lemma:Bernstein}
Let $\brc*{X_i}_{i=1}^n$ be independent random variables in $\brs*{0,1}$ with mean $\E \brs*{X_i}=p$ and variance $\Var\brc*{X}$. Then, with probability $1-\delta$:
\begin{equation}
\label{eq:Bernstein}
\frac{1}{n}\sum_{i=1}^n X_i \le p + \frac{2 \log 1/\delta}{3n} + \sqrt{\frac{2\Var \brc*{X} \log 1/\delta}{n}} \enspace .
\end{equation}
\end{lemma}

Next, let $V=\frac{1}{n}\sum_{i=1}^n\br*{X_i - \br*{\frac{1}{n}\sum_{i=1}^nX_i}}^2$ be the empirical variance of independent random variables $X_i\in\brs*{0,1}$ with mean $p$, and note that 
$V=\frac{1}{n}\sum_{i=1}^n\br*{X_i - p}^2 -  \br*{\frac{1}{n}\sum_{i=1}^nX_i-p}^2 \le \frac{1}{n}\sum_{i=1}^n\br*{X_i - p}^2$. We can thus define the independent random variables $Y_i=\br*{X_i-p}^2$ and bound $\frac{1}{n}\sum_{i=1}^nY_i$ instead. It is clear that $0\le Y_i\le 1$, and their expectations can also be bounded by
\begin{equation*}
\E \brs*{Y_i} = \Var\brc*{X_i} =\E X_i^2-p^2\le 1\cdot\E X_i-p^2=p(1-p) \enspace .
\end{equation*}
The variance of $Y_i$ can be similarly bounded by $\Var\brc*{Y_i} \le \E \brs*{Y_i^2} \le 1\cdot\E \brs*{Y_i} \le p(1-p)$.
Applying Bernstein's Inequality (\ref{eq:Bernstein}) on $Y_i$ can now give a high probability bound on the empirical variance: with probability $1-\delta$,
\begin{align}
\label{eq:bernstein_variance}
V & \le \frac{1}{n}\sum_{i=1}^nY_i 
\le \E \brs*{Y_i} + \frac{2 \log 1/\delta}{3n} + \sqrt{\frac{2\Var \brc*{Y_i} \log 1/\delta}{n}}  \nonumber \\
& \le p(1-p)+ \frac{2\log 1/\delta}{3n} + \sqrt{\frac{2p(1-p) \log 1/\delta}{n}} \nonumber \\
& \stackrel{(*)}{\le} p(1-p)+ \frac{2\log 1/\delta}{3n} + p(1-p) + \frac{\log 1/\delta}{2n} \nonumber \\
& = 2p(1-p)+ \frac{7\log 1/\delta}{6n} \enspace ,
\end{align}
where $(*)$ utilizes the relation $ab\le \frac{1}{2}\br*{a^2+b^2}$ with $a=\sqrt{2p(1-p)}$ and $b=\sqrt{\frac{\log 1/\delta}{n}}$. An important (informal) conclusion of inequality (\ref{eq:bernstein_variance}), is that if the event under which the inequality holds for $\hat{V}_{ij}$ does occur, then the confidence interval around $\hat{p}_{ij}(t-1)$ can be bounded by:

\begin{align}
\label{eq:GiniMotivation}
\sqrt{\frac{6 \hat{V}_{ij}(t-1)\log t}{N_{j}(t-1)}} + \frac{9\log t}{N_{j}(t-1)} 
 & \le \sqrt{\frac{12p_{ij}(1-p_{ij})\log t}{N_{j}(t-1)}+\frac{7\log\frac{1}{\delta}\log t}{N_{j}^2(t-1)}} + \frac{9\log t}{N_{j}(t-1)} \nonumber \\
 & \le \sqrt{\frac{12p_{ij}(1-p_{ij})\log t}{N_{j}(t-1)}} + \frac{\sqrt{7\log\frac{1}{\delta}\log t}}{N_{j}(t-1)} + \frac{9\log t}{N_{j}(t-1)} \enspace .
\end{align}

For $\delta\approx t^{-3}$, the confidence interval is of the form $u_i\sqrt{p_{i}(1-p_i)}+v_i$, for $u_i=O\br*{\sqrt{\frac{\log t}{N_j(t-1)}}}$ and $v_i=O\br*{\frac{\log t}{N_j(t-1)}}$. We should therefore analyze how this kind of parameter perturbations affect the reward function. To do so, we take advantage of the Gini-weighted smoothness, as stated in the following lemma (see Appendix \ref{append:paramSens} for proof):

\begin{restatable}{lemma-rst}{paramSens}
\label{lemma:paramSens}
Let $f(A;x)$ be a monotonic Gini-smooth function, with smoothness parameters $\gamma_\infty$ and $\gamma_g$. Also let $u,v\in\R^L$ be some constant vectors such that $u_{i}\ge0$ and $v_{i}\ge0$.

For any $x,\epsilon\in\brs*{0,1}^{\Narms}$  such that $\epsilon_{i}\le\min\brc*{u_i\sqrt{x_{i}(1-x_i)}+v_i,1-x_{i}}$, the sensitivity of $f(A;x)$ to parameter change can be bounded by
\begin{equation}
\label{eq:param_free_sensitivity_bound}
f(A;x+\epsilon)-f(A;x) \le 3\sqrt{2}\gamma_g \sqrt{\sum_{i\in A} u_i^2} + \gamma_{\infty}\sum_{i\in A} v_i \enspace .
\end{equation}
\end{restatable}

Next, define the low probability events under which some of the variables are not concentrated in their confidence intervals:
\begin{align}
&\HE_t^p\!=\!\brc*{\exists i,j: \abs{\hat{p}_{ij}(t-1)-p_{ij}} > \sqrt{\frac{6\hat{V}_{ij}(t-1)\log t }{N_j(t-1)}} + \frac{9\log t}{N_j(t-1)}} \label{eq:Hp}\\
&\HE_t^V=\brc*{\exists i,j: \hat{V}_{ij}(t-1) > 2p(1-p)+ 3.5\frac{\log t}{N_j(t-1)}} \label{eq:HV}
\end{align}
Also denote $\HE_t=\HE_t^V\cup \HE_t^p$. Intuitively, even though the regret may be large under $\HE_t$, the event cannot occur many times, and we can therefore analyze the regret under the assumption that $\HE_t$ does not occur. We can then bound the regret similarly to inequality (\ref{eq:GiniMotivation}), combined with Lemma \ref{lemma:paramSens}. Formally, we decompose the regret as follows:

\begin{restatable}{lemma-rst}{regretDecomp}
\label{lemma:regretDecomp}
Let $r_i(A;p)$ be Gini-smooth functions with parameters $\gamma_g$,$\gamma_\infty$, and define 
\begin{equation}
c_t(A_t) = c_1 \sqrt{\sum_{j\in A_t}\frac{\log t}{N_j(t-1)}} + c_2 \sum_{j\in A_t}\frac{\log t}{N_j(t-1)} 
\end{equation}
for $c_1=12\sqrt{6}\gamma_g$ and $c_2=34 \gamma_\infty$. The regret of Algorithm \ref{alg:BC-UCB}, when used with $(\oracleApp,\oracleProb)$ approximation oracle, can be bounded by 
\begin{equation}
\label{eq:regretBound}
R(T) \le \E \brs*{\sum_{t=\Narms+1}^T \dr{A_t}\Ind{\dr{A_t}\le \bar\Nitems c_t(A_t)}} +
\Narms \dr{\max}\br*{ 1 + \Nitems\frac{2\pi^2}{3}} .
\end{equation}
\end{restatable}

The proof is in Appendix \ref{append:regretDecomp}. It is interesting to observe that $c_t(A_t)$ is very similar in its form to the confidence interval for the linear combinatorial problem when parameters are independent \citep{combes2015combinatorial}. We have achieved this form of confidence without any independence assumptions, only on the basis of the properties of the reward function. We can therefore adapt the proofs of \citep{kveton2015tight,degenne2016combinatorial} to derive a problem dependent regret bound. 
Since we are interested in bounding the regret of the first term, and due to the initial sampling stage, we assume from this point onward that all of the arms were sampled at least once.

Define $\brc*{a_k}_{k=0}^\infty$ and $\brc*{b_k}_{k=0}^\infty$, two positive decreasing sequences that converge to $0$ and will be determined later, with $b_0=1$. Also define the set $S_t^k=\brc*{j\in A_t: N_j(t-1)\le a_k \frac{g(\Nbatch,\dr{A_t})\log t}{\dr{A_t}^2}}$ for some function $g(\Nbatch,\dr{A_t})$ that will also be determined later, with $S_t^0=A_t$. Intuitively, $S_t^k$ is the set of arms that were chosen on round $t$ and were not sampled enough times. 
Denote the events in which $S_t^k$ contains at least $\Nbatch b_k$ elements, but $S_t^n$ contain less than $\Nbatch b_n$ for any $n<k$, by $\G_t^k = \brc*{\brc*{\abs*{S_t^k}\ge\Nbatch b_k} \cap \brc*{\forall n<k, \abs*{S_t^n}<\Nbatch b_n}}$, and let $\G_t=\cup_{k=1}^\infty \G_t^k$. We show that when $\dr{A_t}\le \bar\Nitems c_t(A_t)$, then $\G_t$ must occur, for the appropriate $g(\Nbatch,\dr{A_t})$. To do so, we first cite a variant of Lemmas 7 and 8 of \citep{degenne2016combinatorial}, with $e=1$, $\Gamma^{(ii)}=1$ and $f(t)=\log t/8$:
\begin{lemma}
\label{lemma:AeventProperties}
Define $\ell = 
\br*{\sum_{k=1}^{k_0}\frac{b_{k-1}-b_k}{a_k} + \frac{b_{k_0}}{a_{k_0}} }$, and let $k_0$ be the smallest index such that $b_{k_0}\le 1/\Nbatch$. Then $\G_t=\cup_{k=1}^{k_0} \G_t^k$. Also, under $\bar\G_t$, the following inequality holds:
\begin{equation}
\sum_{j\in A_t} \frac{\log t}{N_j(t-1)} < \frac{\Nbatch\dr{A_t}^2\ell}{g(\Nbatch,\dr{A_t})} \enspace .
\end{equation}
\end{lemma}

Using this lemma, we can now prove that $\G_t$ must occur (see Appendix \ref{append:GbarImpossible} for proof):
\begin{restatable}{lemma-rst}{GbarImpossible}
\label{lemma:GbarImpossible}
If $g(\Nbatch,\dr{A_t})=\br*{864\gamma_g^2+68\frac{\gamma_\infty\dr{A_t}}{\bar\Nitems}}\bar\Nitems^2\Nbatch\ell$, and if $\dr{A_t}\le \bar\Nitems c_t(A_t)$, then $\G_t$ occurs.
\end{restatable}
A direct result of this lemma is that if $\dr{A_t}\le \bar\Nitems c_t(A_t)$, at least one event $\brc*{\G_t^k}_{k=1}^{k_0}$ occurs. This allows us to further decompose the first term of (\ref{eq:regretBound}), and achieve the final result of Theorem \ref{theorem:dependentRegret}:
\begin{restatable}{lemma-rst}{sumOfSmallGaps}
\label{lemma:sumOfSmallGaps}
The regret from the event $\dr{A_t}\le \bar\Nitems c_t(A_t)$ can be bounded by 
\begin{align}
\sum_{t=\Narms+1}^T& \dr{A_t}\Ind{\dr{A_t}\le \bar\Nitems c_t(A_t)} \nonumber \\ 
& \le \brs*{1728\gamma_g^2\bar\Nitems^2\sum_{j=1}^{\Narms}  \frac{1}{\dr{j,\min}}+68\gamma_\infty \bar\Nitems \sum_{j=1}^{\Narms}  \br*{1 + \log\frac{\dr{j,\max}}{\dr{j,\min}}}}\br*{\sum_{k=1}^{k_0} \frac{a_k}{b_k}} \ell \log T \enspace ,
\end{align}
and if $a_k=b_k=0.2^k$, 
\begin{align}
\label{eq:sumOfSmallGaps}
\sum_{t=\Narms+1}^T &\dr{A_t}\Ind{\dr{A_t}\le \bar\Nitems c_t(A_t)} \nonumber \\
& \le 
\brs*{8640\gamma_g^2\bar\Nitems^2\sum_{j=1}^{\Narms}  \frac{1}{\dr{j,\min}}+340\gamma_\infty \bar\Nitems \sum_{j=1}^{\Narms}  \br*{1 + \log\frac{\dr{j,\max}}{\dr{j,\min}}}}\ceil*{\frac{\log \Nbatch}{1.61}}^2 \log T \enspace .
\end{align}
\end{restatable}

The proof is in Appendix \ref{append:sumOfSmallGaps}. Substituting (\ref{eq:sumOfSmallGaps}) into (\ref{eq:regretBound}) concludes the proof of Theorem \ref{theorem:dependentRegret}. $\hfill\blacksquare$
\vspace{0.3cm}

The proof for the problem independent upper bound of Corollary \ref{corollary:independentRegret} is a direct result of Lemma \ref{lemma:sumOfSmallGaps}. Specifically, the bound can be achieved by decomposing the regret according to Lemma \ref{lemma:regretDecomp}, and then dividing the regret into large gaps ($\dr{A_t}\ge\dr{}$) and small gaps ($\dr{A_t}\le\dr{}$), according to some fixed threshold $\dr{}$. Large gaps are bounded according to Lemma \ref{lemma:sumOfSmallGaps} with $\dr{j,\min}\ge\dr{}$, and small gaps are bounded trivially by $\dr{}T$. The final bound is achieved by optimizing the threshold $\dr{}$. The full proof can be found in Appendix \ref{append:independentRegret}.

\section{Lower Bounds} \label{section:lowerBound}

Although our algorithm enjoys improved upper bounds in comparison to CUCB, it is still interesting to see whether our results are tight in problems where previous bounds are loose. To demonstrate the tightness of our algorithm, we present an instance of the PMC bandit problem, on which our results are tight up to logarithmic factors. We assume throughout the rest of this section that the maximization oracle can output the optimal batch, i.e. has an approximation factor of $\oracleApp=1$, with probability $\oracleProb=1$. This assumption allows us to focus on the difficulty of the problem due to parameter uncertainty and the semi-bandit feedback. We formally state the results in the following proposition:

\begin{proposition}
There exist an instance of the PMC bandit problem with minimal gap $\dr{}$ such that the expected regret of any consistent algorithm\footnote{An algorithm is called consistent if for any problem and any $\alpha>0$, the regret of the algorithm is $R(t)=o(t^\alpha)$ as $t\to\infty$.} is bounded by 
\begin{equation}
\liminf_{t\to\infty} \frac{R(t)}{\log t} \ge \frac{\bar\Nitems^2(\Narms-\Nbatch)}{4\dr{}} \enspace.
\end{equation}
Moreover, for any $T>0$ and $\Narms>\Nbatch$, there exist an instance of the PMC bandit problem such that the expected regret of any algorithm is bounded by 
\begin{equation}
R(T) \ge \frac{1}{20}\bar\Nitems\min\brc*{\sqrt{(\Narms-\Nbatch+1)T},T} \enspace.
\end{equation}
\end{proposition}

\begin{proof}

Consider the following PMC bandit problem: fix the first $\Nbatch-1$ arms to be empty sets, that is, $X_{ij}=0$ for any $(i,j)\in\brs*{\Nitems}\times\brs*{\Nbatch-1}$, which also implies $p_{ij}=0$. For the rest of the arms $j\in\brc*{\Nbatch,\dots,\Narms}$, we force all of the items to be identically distributed, i.e., $X_{ij}=X_j$ and $p_{ij}=p_{j}$.
We also fix the action set to be $\mathcal{A}=\brc*{A_j}_{j=1}^{\Narms-\Nbatch+1}$, where $A_j=\brc*{1,\dots,\Nbatch-1,j+\Nbatch-1}$ contains all of the $\Nbatch-1$ arms with zero reward plus an additional arm $j+\Nbatch-1$. The expected reward when choosing an action $A_j$ is thus $\bar\Nitems p_j\triangleq \mu_j$, and the problem is equivalent to a $\br*{\Narms-\Nbatch+1}$-armed bandit problem with arm distribution $Y_j=\sum_i w_i X_{j+\Nbatch-1}=\bar\Nitems X_{j+\Nbatch-1}$.
In order to prove the problem dependent regret bound, let $p_j=\frac{1}{2}-\epsilon, \forall j\in\brc*{\Nbatch,\dots,\Narms-1}$ and $p_\Narms=\frac{1}{2}$, or equivalently, $\mu_j=\frac{\bar\Nitems}{2}-\bar\Nitems\epsilon, \forall j\in\brs*{\Narms-\Nbatch}$ and $\mu_{\Narms-\Nbatch+1}=\frac{\bar\Nitems}{2}$. The gaps of the problem are thus $\dr{j}=\bar\Nitems\epsilon\triangleq \dr{}$ for any arm $j\in\brs*{\Narms-\Nbatch}$. For any consistent MAB algorithm, the expected regret of the algorithm can be lower bounded by \citep{lai1985asymptotically}:
\begin{equation}
\label{eq:laiRobbins}
    \lim\inf_{t\to\infty} \frac{R(t)}{\log t} \ge \sum_{j=1}^{\Narms-\Nbatch} \frac{\dr{j}}{\kl\br*{Y_j,Y_{\Narms-\Nbatch+1}}}\enspace .
\end{equation}
The KL divergence can be directly bounded by
\begin{equation*}
    \kl\br*{Y_j,Y_{\Narms-\Nbatch+1}} = \klBin\br*{p_j,p_{\Narms-\Nbatch+1}} = 
    \klBin\br*{\frac{1}{2}-\epsilon,\frac{1}{2}}
    \stackrel{(*)}{\le} 4\epsilon^2 = 4\frac{\dr{}^2}{\bar\Nitems^2} \enspace,
\end{equation*}
where $(*)$ is due to the relation $\klBin(p,q)\le \frac{\br*{p-q}^2}{q(1-q)}$ \citep{csiszar2006context}. Substituting back into (\ref{eq:laiRobbins}) yields the first part of the proposition:
\begin{equation*}
    \lim\inf_{t\to\infty} \frac{R(t)}{\log t} 
    \ge \sum_{j=1}^{\Narms-\Nbatch} \frac{\dr{j}}{\kl\br*{Y_j,Y_{\Narms-\Nbatch+1}}}
    \ge \sum_{j=1}^{\Narms-\Nbatch} \frac{\dr{}}{4\dr{}^2/\bar\Nitems^2}
    = \frac{\bar\Nitems^2\br*{\Narms-\Nbatch}}{4\dr{}}\enspace .
\end{equation*}
Due to the fact that all of the items have the same distribution, our problem is equivalent to an $(\Narms-\Nbatch+1)$ MAB problem with arms in $[0,1]$ scaled by a factor of $\bar\Nitems$. Thus, the problem independent lower bound from \citep{auer2002nonstochastic} can also be applied to this problem, and is scaled by the same factor $\bar\Nitems$:
\begin{equation*}
R(T) \ge \frac{1}{20}\bar\Nitems\min\brc*{\sqrt{(\Narms-\Nbatch+1)T},T} \enspace.
\end{equation*}
\end{proof}

We remark that throughout the proof, we assumed nothing about the correlation between different arms, and thus the bound cannot be improved by assuming this kind of independence. Nevertheless, if items are assumed to be independent, the lower bounds can be drastically improved. We will not tackle the problem of independent items in this paper, but leave it to future work.


\section{Summary} \label{section:summary}

In this work, we introduced BC-UCB, a CMAB algorithm that utilizes Bernstein-type confidence intervals. We defined a new smoothness criterion called Gini-weighted smoothness, and showed that it allows us to derive tighter regret bounds in many interesting problems. We also presented matching lower bounds for such a problem, and thus demonstrated the tightness of our algorithm.

We believe that our concepts can be applied to derive tighter bounds in many interesting settings. Specifically, our analysis includes the PMC bandit problem, that has a central place in the areas of ranked recommendations and influence maximization. We also believe that our results could be extended to the frameworks of cascading bandits and probabilistically triggered arms, but leave this for future work.

Another possible direction involves analyzing specific arm distributions - in our framework, we assumed nothing about the arms' distribution except for its domain, and thus could only take into account very weak concentration properties, and specifically the concentration properties around the edges of the domain. If additional information about the distribution of the arms is present, it should be possible to leverage such information to design more sophisticated smoothness criteria. Such criteria could take into account tighter concentration properties of the arms' distribution, and thus lead to tighter regret bounds. 

Finally, we remark that the lower bounds were possible to derive only since we have required the algorithm to support any arbitrary choice of action set $\mathcal{A}$. For the PMC problem, previous work shows that when $\mathcal{A}$ contains any subset of fixed size $\Nbatch$, the regret bounds can be significantly improved \citep{kveton2015cascading}. It is interesting to see if our technique can be used in this setting to extend these results and derive tighter bounds for any Gini-smooth function.

\acks{The authors thank Asaf Cassel and Esther Derman for their helpful comments on the manuscript.}

\bibliography{references}

\appendix

\newpage

\section{Proof of Lemma \ref{lemma:paramSens}}
\label{append:paramSens}
\paramSens* 
\begin{proof}

First, we define the functions
\begin{equation}
g(z) = \int_{0}^{z}\frac{dy}{\sqrt{y(1-y)}}
\quad , \quad
h(z) = \int_{0}^{z}\frac{dy}{\min\brc*{\sqrt{y},\sqrt{1-y}}} \enspace .
\end{equation}
We note that $g(z)$ is well defined for $z\in[0,1]$, since the function $1/\sqrt{y}$ is integrable near $y=0$ and $1/\sqrt{1-y}$ is continuous, so the product $\frac{1}{\sqrt{y(1-y)}}$ is integrable near $z=0$. Symmetrically, the function is also integrable near $y=1$. $h(z)$ can be explicitly written as

\begin{equation}
    h(z) = 
     \begin{cases}
       2\sqrt{z}  &, z\le \frac{1}{2}\\
       2\sqrt{2}-2\sqrt{1-z}  &, z\ge \frac{1}{2} \\
     \end{cases}
\end{equation}

The two functions are closely related: observe that $h'(z) \le g'(z) \le \sqrt{2}h'(z)$ with $g(0)=h(0)=0$, and thus $h(z) \le g(z) \le \sqrt{2}h(z)$. In addition, $g'(z)>0$, and therefore the function is strictly monotonically increasing, so its inverse $g^{-1}$ is well defined. Finally, the relation between the derivatives also yields the property $g(z_2)-g(z_1) \le \sqrt{2}\br*{h(z_2)-h(z_1)}$ for any $z_1\le z_2$ in $[0,1]$.

Next, we bound $f(A;x+\delta)-f(A;x)$, for any $\delta$ such that  $\delta_i\le\min\brc*{u_i\sqrt{x_{i}(1-x_i)},1-x_{i}}$. The bound can be achieved using the gradient theorem:

\begin{equation}
\label{eq:gradTheorem}
    f(A;x+\delta)-f(A;x) = \int_x^{x+\delta} \nabla f(A;y)\cdot dy = \int_0^1 \sum_{i\in A} \frac{\partial f(A;r(t))}{\partial x_i}r'_i(t) dt \enspace ,
\end{equation}

for a parameterization $r(t)$ such that $r_i(0)=x_i$ and $r_i(1)=x_i+\delta_i$. Specifically, we choose the parameterization to be

\begin{equation*}
    r_i(t) = g^{-1} \br*{ \brs*{g(x_i+\delta_i)-g(x_i)}t + g(x_i) } \enspace ,
\end{equation*}

and thus its gradient is 

\begin{equation*}
    r'_i(t) =\frac{g(x_i+\delta_i)-g(x_i)}{g'(r_i(t))}  
    = \br*{g(x_i+\delta_i)-g(x_i)}\sqrt{r_i(t)(1-r_i(t))} \enspace .
\end{equation*}

Substituting back into (\ref{eq:gradTheorem}) yields
{\small
\begin{equation}
\label{eq:gradTheoremCalc}
\begin{aligned}
f(A;x+\delta)-f(A;x)
& = \int_0^1 \sum_{i\in A} \br*{g(x_i+\delta_i)-g(x_i)} \frac{\partial f(A;r(t))}{\partial x_i}\sqrt{r_i(t)(1-r_i(t))} dt \\
& \le \int_0^1 \sqrt{\sum_{i\in A} \br*{g(x_i+\delta_i)-g(x_i)}^2} \sqrt{ \sum_{i\in A} \br*{\frac{\partial f(A;r(t))}{\partial x_i}}^2 r_i(t)(1-r_i(t))} dt \\
& \le \gamma_g\sqrt{\sum_{i\in A} \br*{g(x_i+\delta_i)-g(x_i)}^2} \enspace .
\end{aligned}
\end{equation}
}
The first inequality is due to Cauchy Schwarz and the second uses the definition of the Gini-weighted smoothness (\ref{eq:smoothDef}). All that's left is bounding the differences between different values of $g$. Assume w.l.o.g. that $x_i\neq 0,1$, since otherwise $\delta_i=0$ and the difference is $0$. To calculate the bound, we exploit the relation between $g$ and $h$, and calculate differences over $h$ in three different cases:

\begin{itemize}
    \item If $x_i,x_i+\delta_i\le\frac{1}{2}$, then 
    \begin{equation*}
    \begin{aligned}
    h(x_i+\delta_i)-h(x_i) 
    &= 2\sqrt{x_i+\delta_i}-2\sqrt{x_i}
    = 2\sqrt{x_i}\br*{\sqrt{1+\frac{\delta_i}{x_i}}-1}
    \le 2\sqrt{x_i}\frac{\delta_i}{2x_i} \\
    & =\frac{\delta_i}{\sqrt{x_i}} \le u_i \enspace,
    \end{aligned}
    \end{equation*}
    where the first inequality is since $\sqrt{1+a}\le 1+\frac{a}{2}$ and the second is due to $\delta_i\le u_i\sqrt{x_i}$.
    
    \item If $x_i,x_i+\delta_i\ge\frac{1}{2}$, then 
    \begin{equation*}
    \begin{aligned}
    h(x_i+\delta_i)-h(x_i) 
    & = 2\sqrt{1-x_i}-2\sqrt{1-x_i-\delta_i}
    = 2\sqrt{1-x_i}\br*{1 - \sqrt{1-\frac{\delta_i}{1-x_i}}} \\
    & \le 2\sqrt{1-x_i}\frac{\delta_i}{1-x_i}
    = 2\frac{\delta_i}{\sqrt{1-x_i}}
    \le 2u_i \enspace,
    \end{aligned}
    \end{equation*}
    where the first inequality is due to $1-\sqrt{1-a} \le a$ for $a\in[0,1]$ and the second is since $\delta_i\le u_i\sqrt{1-x_i}$.
    
    \item If $x_i\le\frac{1}{2}$ and $x_i+\delta_i\ge\frac{1}{2}$, then 
    \begin{equation*}
    \begin{aligned}
    h(x_i+\delta_i)-h(x_i) 
    & = \br*{h(x_i+\delta_i)-h(\frac{1}{2})}+\br*{h(\frac{1}{2})-h(x_i)} \\
    & \stackrel{(1)}{\le} 2 \frac{x_i+\delta_i-1/2}{\sqrt{1-1/2}} + \frac{1/2- x_i}{\sqrt{x_i}} \\
    & \stackrel{(2)}{\le} 2 \frac{\delta_i}{\sqrt{x_i}} + \frac{\delta_i}{\sqrt{x_i}} 
    \le 3u_i \enspace.
    \end{aligned}
    \end{equation*}
    $(1)$ uses the bound calculated for the first two cases and $(2)$ takes advantage of the relation $x_i\le\frac{1}{2}$ for the first term and $x_i+\delta_i\ge \frac{1}{2}$ for the second one.
\end{itemize}

To summarize, for any $x_i,x_i+\delta_i\in[0,1]$, we have 

\begin{equation*}
    \begin{aligned}
    g(x_i+\delta_i)-g(x_i) 
    \le \sqrt{2}\brs*{h(x_i+\delta_i)-h(x_i) }
    \le 3\sqrt{2}u_i \enspace .
    \end{aligned}
    \end{equation*}

Substituting back into (\ref{eq:gradTheoremCalc}) leads to the bound

\begin{equation}
\label{eq:gradTheoremCalcEnd}
\begin{aligned}
f(A;x+\delta)-f(A;x)\le \gamma_g\sqrt{\sum_{i\in A} \br*{g(x_i+\delta_i)-g(x_i)}^2}
\le 3\sqrt{2} \gamma_g\sqrt{\sum_{i\in A} u_i^2} \enspace.
\end{aligned}
\end{equation}

Next, we aim to bound $f(A;x+\epsilon)-f(A;x+\delta)$. Fortunately, this term can be easily bounded using the gradient theorem as follows:

\begin{align}
\label{eq:offsetBound}
    f(A;x+\epsilon)-f(A;x+\delta) 
    & = \int_{x+\delta}^{x+\epsilon} \nabla f(A;y)\cdot dy \nonumber \\
    &\le \sup_y \brc*{\norm{\nabla f(A;y)}_{\infty}} \sum_{i\in A} \br*{\epsilon_i-\delta_i} 
    \le \gamma_\infty \sum_{i\in A} v_i
\end{align}

and combining (\ref{eq:gradTheoremCalcEnd}) and (\ref{eq:offsetBound}) concludes the proof.

\end{proof}

\section{Proof of Lemma \ref{lemma:regretDecomp}}
\label{append:regretDecomp}

\regretDecomp*
\begin{proof}

By the definition of the expected regret,
\begin{align}
\label{eq:regretDecompProof}
R(T) & = \appFac r_{\max} T - \sum_{t=1}^T \E\brs*{r\br*{A_t;p}} 
= \sum_{t=1}^T \E\brs*{\dr{A_t}} - \oracleApp(1-\oracleProb)r_{\max}T \nonumber \\
& \le \Narms \max_A\dr{A} + 
\E \brs*{\sum_{t=\Narms+1}^T \dr{A_t}\Ind{\HE_t}}  + \nonumber \\
& \quad \;\,
\E \brs*{\sum_{t=\Narms+1}^T \dr{A_t}\Ind{\bar{\HE}_t}}  - \oracleApp(1-\oracleProb)r_{\max}T , 
\end{align}
where $\HE_t=\HE_t^p\cup \HE_t^V$ is the event in which the concentration bounds fail to hold at time $t$ and $\HE_t^p, \HE_t^V$ are defined in Equations (\ref{eq:Hp}) and (\ref{eq:HV}). $\bar{\HE}_t$ is the complementary event to $\HE_t$. 
The first term in (\ref{eq:regretDecompProof}) is the regret due to the initial sampling, that forces sampling each base arm at least once, on which we apply the bound $\dr{A}\le \dr{\max}$. For the rest of the terms, the initial sampling stage allows us to assume that all of the arms were sampled at least once.
The second term can be bounded using Empirical Bernstein (Lemma \ref{lemma:emp_bernstein}) and Bernstein inequality (\ref{eq:bernstein_variance}) as follows:

{\small
\begin{align}
\label{eq:regret_under_H}
\E& \brs*{\sum_{t=\Narms+1}^T \dr{A_t}\Ind{\HE_t}} \nonumber \\
&\le  \E \brs*{\dr{\max}\sum_{t=\Narms+1}^T \Ind{\HE_t}}\nonumber \\
& \stackrel{(1)}{\le} \dr{\max} \sum_{t=\Narms+1}^T \br*{\Pr\brc*{\HE_t^p} + \Pr\brc*{\HE_t^V}} \nonumber \\
& = \dr{\max} \sum_{t=\Narms+1}^T 
\Pr\brc*{\exists i,j: \abs{\hat{p}_{ij}(t-1)-p_{ij}} > \sqrt{\frac{6\log t \hat{V}_{ij}(t-1)}{N_j(t-1)}} + \frac{9\log t}{N_j(t-1)}} + \nonumber \\
& \;\;\;\;\; \dr{\max} \sum_{t=\Narms+1}^T 
\Pr\brc*{\exists i,j: \hat{V}_{ij}(t-1) > 2p_{ij}(1-p_{ij})+ 3.5\frac{\log t}{N_j(t-1)}} \nonumber \\
& \stackrel{(2)}{\le} \dr{\max} \sum_{t=\Narms+1}^T\sum_{i=1}^\Nitems \sum_{j=1}^\Narms 
\Pr\brc*{\abs{\hat{p}_{ij}(t-1)-p_{ij}} > \sqrt{\frac{6\log t \hat{V}_{ij}(t-1)}{N_j(t-1)}} + \frac{9\log t}{N_j(t-1)}} + \nonumber \\
& \;\;\;\;\; \dr{\max} \sum_{t=\Narms+1}^T\sum_{i=1}^\Nitems \sum_{j=1}^\Narms 
\Pr\brc*{\hat{V}_{ij}(t-1) > 2p_{ij}(1-p_{ij})+ 3.5\frac{\log t}{N_j(t-1)}} \nonumber \\
& \stackrel{(3)}{\le} \dr{\max} \sum_{t=\Narms+1}^T\sum_{i=1}^\Nitems \sum_{j=1}^\Narms \sum_{s=1}^t \Pr\brc*{\abs{\hat{p}_{ij}(t-1)-p_{ij}} > \sqrt{\frac{6\log t \hat{V}_{ij}(t-1)}{s}} + \frac{9\log t}{s},N_j(t-1)=s} + \nonumber \\
&\;\;\;\;\; \dr{\max} \sum_{t=\Narms+1}^T\sum_{i=1}^\Nitems \sum_{j=1}^\Narms \sum_{s=1}^t \Pr\brc*{\hat{V}_{ij}(t-1) > 2p_{ij}(1-p_{ij})+ 3.5\frac{\log t}{s},N_j(t-1)=s} \nonumber \\
& \stackrel{(4)}{\le} \dr{\max}\sum_{t=\Narms+1}^T\sum_{i=1}^\Nitems \sum_{j=1}^\Narms \sum_{s=1}^t 3 t^{-3} + \dr{\max}\sum_{t=\Narms+1}^T\sum_{i=1}^\Nitems \sum_{j=1}^\Narms \sum_{s=1}^t t^{-3} \nonumber \\
& = \dr{\max}\sum_{t=\Narms+1}^T\sum_{i=1}^\Nitems \sum_{j=1}^\Narms 3 t^{-2} + \dr{\max}\sum_{t=\Narms+1}^T\sum_{i=1}^\Nitems \sum_{j=1}^\Narms t^{-2} \nonumber \\
& \le 4\dr{\max}\Nitems\Narms\frac{\pi^2}{6}
\end{align}
}
Inequalities $(1)-(3)$ are simple union bounds: $(1)$ is a union bound on $\HE_t^p$ and $\HE_t^V$ and on different times. $(2)$ is on different functions $i\in\brs*{\Nitems}$ and arms $j\in\brs*{\Narms}$. $(3)$ is a union bound on different values of $N_j(t-1)$, ranging in $s\in\brs*{t}$. $(4)$ uses Empirical Bernstein (Lemma \ref{lemma:emp_bernstein}) with $x=3\log t$ for $\hat{p}_{ij}(t-1)$ and Bernstein's inequality (\ref{eq:bernstein_variance}) for $\hat{V}_{ij}(t-1)$ with $\delta = t^{-3}$. 

Next, we bound the difference $\dr{A_t}$ under the event $\bar{\HE}_t$. Note that under $\bar\HE_t^p$, and for any $i,j$, it holds that $p_{ij} \le q_{ij}(t)$ and 
\begin{equation*}
q_{ij}(t) \le p_{ij} + 2\br*{\sqrt{\frac{6\log t \hat{V}_{ij}(t-1)}{N_j(t-1)}} + \frac{9\log t}{N_j(t-1)}} \enspace .
\end{equation*}
In addition, since $\bar\HE_t^V$ also occurs, we can further bound $q_{ij}$ by
\begin{align}
\label{eq:q_upper_bound}
q_{ij}(t) 
&\le p_{ij} + 2\br*{\sqrt{\frac{6\log t \br*{ 2p_{ij}(1-p_{ij})+ 3.5\frac{\log t}{N_j(t-1)}}}{N_j(t-1)}} + \frac{12\log t}{N_j(t-1)}} \nonumber \\
& \le p_{ij} + 2\br*{\sqrt{\frac{12\log t}{N_j(t-1)}p_{ij}(1-p_{ij})} + \sqrt{21}\frac{\log t}{N_j(t-1)} + \frac{12\log t}{N_j(t-1)}} \nonumber \\
& \le p_{ij} + 4\sqrt{3} \sqrt{\frac{\log t}{N_j(t-1)}p_{ij}(1-p_{ij})} + \frac{34\log t}{N_j(t-1)} \enspace,
\end{align}
where we used the inequality $\sqrt{x+y}\le\sqrt{x}+\sqrt{y}$. 

Notice that $q_{ij}(t)$ is of the form $q_{ij}(t)-p_{ij}\le u_j\sqrt{p_{ij}(1-p_{ij})}+v_j$, for $u_j=4\sqrt{3} \sqrt{\frac{\log t}{N_j(t-1)}}$ and $v_j=34\frac{\log t}{N_j(t-1)}$. We can thus apply Lemma \ref{lemma:paramSens} to bound the difference between the optimistic reward with the UCB index $q$ and the real reward parameters $p$. Combined with the properties of the approximation oracle, it is possible to estimate the error between the optimal action and the selected action, under the assumption that the oracle succeeded in his approximation (an event which we denote by $\F_t$)

\begin{equation*}
\begin{aligned}
\oracleApp r(A^*;p) 
& \stackrel{(1)}{\le} \oracleApp r(A^*;q)  
\stackrel{(2)}{\le} r(A_t;q)
= \sum_{i=1}^\Nitems w_i r_i(A_t;q) \\
& \stackrel{(3)}{\le}\sum_{i=1}^\Nitems w_i \br*{r_i(A_t;p) + 12\sqrt{6}\gamma_g\sqrt{\sum_{j\in A_t} \frac{\log t}{N_j(t-1)}}+34\gamma_\infty\sum_{j\in A_t}\frac{\log t}{N_j(t-1)}} \\
& = r(A_t;p) + \bar\Nitems  \br*{ 12\sqrt{6}\gamma_g\sqrt{\sum_{j\in A_t}\frac{\log t}{N_j(t-1)}}+34\gamma_\infty\sum_{j\in A_t}\frac{\log t}{N_j(t-1)}} \\
& = r(A_t;p) + \bar\Nitems c_t(A_t)
\end{aligned}
\end{equation*}
where $(1)$ is due to the monotonicity of $r$ and $p\le q$, $(2)$ is from the properties of the approximation oracle and $(3)$ applies Lemma \ref{lemma:paramSens}. Consequentially, the sub-optimality gap of action $A_t$ under the events $\bar\HE_t$ and $\F_t$ can be bounded by
\begin{equation*}
\dr{A_t} = \oracleApp r_{\max} -r(A_t;p) 
=  \oracleApp r(A^*;p) -r(A_t;p)
\le \bar\Nitems c_t(A_t) \enspace ,
\end{equation*}

which in turn, allows us to bound the third term of (\ref{eq:regretDecompProof}) by
\begin{align}
\label{eq:delta_under_Hbar}
\E \brs*{\sum_{t=\Narms+1}^T \dr{A_t}\Ind{\bar{\HE}_t}} 
&\le \E \brs*{\sum_{t=\Narms+1}^T \dr{A_t}\Ind{\bar{\HE}_t,\F_t}} +  \E \brs*{\sum_{t=\Narms+1}^T \dr{A_t}\Ind{\bar\F_t}} \nonumber \\
& \le \E \brs*{\sum_{t=\Narms+1}^T \dr{A_t}\Ind{\dr{A_t} \le \bar\Nitems c_t(A_t)}} +  (1-\oracleProb)\oracleApp r_{\max}T \enspace .
\end{align}

Substituting equations (\ref{eq:regret_under_H}) and (\ref{eq:delta_under_Hbar}) into (\ref{eq:regretDecompProof}) leads to the desired result:
\begin{equation*}
R(T) \le 
\Narms\dr{\max}  + 
2\dr{\max}\Nitems\Narms\frac{\pi^2}{6} +
\E \brs*{\sum_{t=\Narms+1}^T \dr{A_t}\Ind{\dr{A_t}\le \bar\Nitems c_t(A_t)}}
\end{equation*}
\end{proof}

\section{Proof of Lemma \ref{lemma:GbarImpossible}}
\label{append:GbarImpossible}

\GbarImpossible*
\begin{proof}

Assume in contradiction that both $\dr{A_t}\le \bar\Nitems c_t(A_t)$ and $\bar\G_t$ occur. From the definition of $c_t(A_t)$ and Lemma \ref{lemma:AeventProperties}, we get
\begin{equation*}
\begin{aligned}
\dr{A_t} & \le \bar\Nitems c_t(A_t) 
= c_1\bar\Nitems\sqrt{ \sum_{j\in A_t}\frac{\log t}{N_j(t-1)}} + \bar\Nitems c_2 \sum_{j\in A_t}\frac{\log t}{N_j(t-1)} \\
& = 12\sqrt{6}\gamma_g\bar\Nitems\sqrt{ \sum_{j\in A_t}\frac{\log t}{N_j(t-1)}} + 34\gamma_\infty \bar\Nitems \sum_{j\in A_t}\frac{\log t}{N_j(t-1)}\\
& < 12\sqrt{6}\gamma_g\bar\Nitems\sqrt{ \frac{\Nbatch\ell\dr{A_t}^2}{g(\Nbatch,\dr{A_t})}} 
+34\gamma_\infty  \bar\Nitems \frac{\Nbatch\ell\dr{A_t}^2}{g(\Nbatch,\dr{A_t})} \\
&= \frac{12\sqrt{6}\gamma_g}{\sqrt{864\gamma_g^2+68\frac{\gamma_\infty\dr{A_t}}{\bar\Nitems}}} \dr{A_t}
+ \frac{34\gamma_\infty }{\bar\Nitems \br*{864\gamma_g^2+68\frac{\gamma_\infty\dr{A_t}}{\bar\Nitems}}} \dr{A_t}^2 \\
& = \frac{12\sqrt{6}\gamma_g\sqrt{864\gamma_g^2+68\frac{\gamma_\infty\dr{A_t}}{\bar\Nitems}} + 34\frac{\gamma_\infty\dr{A_t}}{\bar\Nitems} }{\br*{864\gamma_g^2+68\frac{\gamma_\infty\dr{A_t}}{\bar\Nitems}}} \dr{A_t}\\
& \stackrel{(*)}{\le} \frac{864\gamma_g^2+34\frac{\gamma_\infty\dr{A_t}}{\bar\Nitems} + 34\frac{\gamma_\infty\dr{A_t}}{\bar\Nitems} }{\br*{864\gamma_g^2+68\frac{\gamma_\infty\dr{A_t}}{\bar\Nitems}}} \dr{A_t}\\
&=\dr{A_t} \enspace ,
\end{aligned}
\end{equation*}
where $(*)$ is due to $\sqrt{a(a+b)}\le a + \frac{b}{2}$. We got $\dr{A_t}<\dr{A_t}$, which is a contradiction, and therefore, if $\dr{A_t}\le \bar\Nitems c_t(A_t)$, then $\bar\G_t$ cannot occur.
\end{proof}

\section{Proof of Lemma \ref{lemma:sumOfSmallGaps}}
\label{append:sumOfSmallGaps}
\sumOfSmallGaps*
\begin{proof}

Recall that  $\G_t^k = \brc*{\brc*{\abs*{S_t^k}\ge\Nbatch b_k} \cap \brc*{\forall n<k, \abs*{S_t^n}<\Nbatch b_n}}$. Specifically, if $\G_t^k$ occurs, then at least $\Nbatch b_k$ arms were sampled less than $a_k\frac{g(\Nbatch,\dr{A_t})\log t}{\dr{A_t}^2}$ times. 
Denote the sub-event of $\G_t^k$ in which $j$ is one of these arms by $\G_t^{k,j}=\G_t^k\cap\brc*{j\in A_t, N_j(t-1)\le a_k\frac{g(\Nbatch,\dr{A_t})\log t}{\dr{A_t}^2}}$. Therefore, if $\G_t^k$ occurs, there are at least  $\Nbatch b_k$ arms for which $\G_t^{k,j}$ occur, and 
\begin{equation*}
    \Ind{\G_t^{k}} \le \frac{1}{\Nbatch b_k}\sum_{j=1}^{\Narms} \Ind{\G_t^{k,j}} \enspace .
\end{equation*}

Next, we apply Lemma \ref{lemma:GbarImpossible}, and bound the regret from the event $\dr{A_t}\le \bar\Nitems c_t(A_t)$ by
\begin{equation*}
\begin{aligned}
    \sum_{t=\Narms+1}^T \dr{A_t}\Ind{\dr{A_t}\le \bar\Nitems c_t(A_t)} 
    \le \sum_{t=\Narms+1}^T \sum_{k=1}^{k_0}\dr{A_t}\Ind{\G_t^k}
    \le \sum_{t=\Narms+1}^T \sum_{k=1}^{k_0}\sum_{j=1}^{\Narms} \frac{\dr{A_t}}{\Nbatch b_k}\Ind{\G_t^{k,j}} \enspace .
\end{aligned}
\end{equation*}
Denote the number of possible positive gaps of batches that contain arm $j$ by $D_j$, and assume $\dr{j,1}>,\dots>\dr{j,D_j}=\dr{j,\min}>0$ with $\dr{j,0}=\infty$. We decompose the regret as 

\begin{align}
\label{eq:regretBoundDecomposeA}
    &\sum_{t=\Narms+1}^T \dr{A_t} \Ind{\dr{A_t}\le \bar\Nitems c_t(A_t)} 
    \le \sum_{t=\Narms+1}^T \sum_{k=1}^{k_0}\sum_{j=1}^{\Narms}\sum_{n=1}^{D_j} \frac{\dr{j,n}}{\Nbatch b_k}\Ind{\G_t^{k,j},\dr{A_t}=\dr{j,n}} \nonumber \\
    & \le \sum_{t=\Narms+1}^T \sum_{k=1}^{k_0}\sum_{j=1}^{\Narms}\sum_{n=1}^{D_j} \frac{\dr{j,n}}{\Nbatch b_k}\Ind{j\in A_t, N_j(t-1)\le a_k\frac{g(\Nbatch,\dr{A_t})\log t}{\dr{j,n}^2},\dr{A_t}=\dr{j,n}} \enspace .
\end{align}

Denote $\theta_{k}= 864\gamma_g^2\bar\Nitems^2\Nbatch\ell a_k \log T$ and $\mu_{k}= 68\gamma_\infty\bar\Nitems\Nbatch\ell a_k \log T$. Under these notations
{\small
\begin{equation*}
\begin{aligned}
   & \sum_{t=\Narms+1}^T  \sum_{n=1}^{D_j}  \dr{j,n}\Ind{j\in A_t, N_j(t-1)\le a_k\frac{g(\Nbatch,\dr{A_t})\log t}{\dr{j,n}^2},\dr{A_t}=\dr{j,n}} \\
    & \stackrel{(1)}{\le} \sum_{t=\Narms+1}^T \sum_{n=1}^{D_j} \dr{j,n}\Ind{j\in A_t, N_j(t-1)\le \frac{\theta_{k}}{\dr{j,n}^2}+\frac{\mu_{k}}{\dr{j,n}},\dr{A_t}=\dr{j,n}} \\
    & \stackrel{(2)}{\le} \sum_{t=\Narms+1}^T \sum_{n=1}^{D_j} \sum_{p=1}^n \dr{j,n}\Ind{j\in A_t, N_j(t-1) \in\left( \frac{\theta_{k}}{\dr{j,p-1}^2}+\frac{\mu_{k}}{\dr{j,p-1}}, \frac{\theta_{k}}{\dr{j,p}^2}+\frac{\mu_{k}}{\dr{j,p}} \right] ,\dr{A_t}=\dr{j,n}} \\
    & \stackrel{(3)}{\le} \sum_{t=\Narms+1}^T \sum_{n=1}^{D_j} \sum_{p=1}^n \dr{j,p}\Ind{j\in A_t, N_j(t-1) \in\left( \frac{\theta_{k}}{\dr{j,p-1}^2}+\frac{\mu_{k}}{\dr{j,p-1}}, \frac{\theta_{k}}{\dr{j,p}^2}+\frac{\mu_{k}}{\dr{j,p}} \right] ,\dr{A_t}=\dr{j,n}} \\
    & \stackrel{(4)}{\le} \sum_{t=\Narms+1}^T \sum_{n=1}^{D_j} \sum_{p=1}^{D_j} \dr{j,p}\Ind{j\in A_t, N_j(t-1) \in\left( \frac{\theta_{k}}{\dr{j,p-1}^2}+\frac{\mu_{k}}{\dr{j,p-1}}, \frac{\theta_{k}}{\dr{j,p}^2}+\frac{\mu_{k}}{\dr{j,p}} \right] ,\dr{A_t}=\dr{j,n}} \\
    & \stackrel{(5)}{\le} \sum_{t=\Narms+1}^T \sum_{p=1}^{D_j} \dr{j,p}\Ind{j\in A_t, N_j(t-1) \in\left( \frac{\theta_{k}}{\dr{j,p-1}^2}+\frac{\mu_{k}}{\dr{j,p-1}}, \frac{\theta_{k}}{\dr{j,p}^2}+\frac{\mu_{k}}{\dr{j,p}} \right] ,\dr{A_t}>0} \\
    & \stackrel{(6)}{\le} \frac{\theta_{k}}{\dr{j,1}} +\mu_k + \sum_{p=2}^{D_j} \dr{j,p}\theta_{k} \br*{\frac{1}{\dr{j,p}^2}-\frac{1}{\dr{j,p-1}^2}} + \sum_{p=2}^{D_j} \dr{j,p}\mu_{k} \br*{\frac{1}{\dr{j,p}}-\frac{1}{\dr{j,p-1}}} \\
    & = \frac{\theta_{k}}{\dr{j,D_j}} + +\mu_k + \theta_{k} \sum_{p=1}^{D_j-1}  \frac{\dr{j,p}-\dr{j,p+1}}{\dr{j,p}^2} + \mu_{k} \sum_{p=1}^{D_j-1}  \frac{\dr{j,p}-\dr{j,p+1}}{\dr{j,p}}\\
    & \le \frac{\theta_{k}}{\dr{j,D_j}} + \mu_k +\theta_{k}\int_{\dr{j,D_j}}^{\dr{j,1}} x^{-2} dx + \mu_{k}\int_{\dr{j,D_j}}^{\dr{j,1}} x^{-1} dx \\
    &\le  \frac{2\theta_{k}}{\dr{j,D_j}} +\mu_{k} + \mu_{k}\log \frac{\dr{j,1}}{\dr{j,D_j}} \\
    & =  \frac{2\theta_{k}}{\dr{j,\min}} + \mu_{k}\br*{1 + \log \frac{\dr{j,\max}}{\dr{j,\min}}} \enspace .
\end{aligned}
\end{equation*}
}
We added numbering through the first lines, as changes between consecutive lines can be hard to discern. In $(1)$, we bounded $a_k g(\Nbatch,\dr{A_t})\log t \le \theta_{k} + \mu_k \dr{A_t}$. In $(2)$, we divided the interval $N_j(t-1)\le \frac{\theta_{k}}{\dr{j,n}^2}$ into non overlapping sub-interval and used union bound. Next, we replaced $\dr{j,n}$ by larger $\dr{j,p}$ in $(3)$, and extended the internal sum to $D_j$ in $(4)$. Finally, in $(5)$ we replaced the sum over specific positive gaps by the event $\dr{A_t}>0$ and in $(6)$ we bounded the maximal number of times that each of the indicators can be nonzero by he length of the interval. The rest of the lines involve reordering and bounding a summation by an integral. 
Substituting back into (\ref{eq:regretBoundDecomposeA}) gives the first desired result:

\begin{equation*}
\begin{aligned}
    \sum_{t=\Narms+1}^T & \dr{A_t}\Ind{\dr{A_t}\le \bar\Nitems c_t(A_t)} 
     \le \sum_{k=1}^{k_0}\sum_{j=1}^{\Narms} \frac{1}{\Nbatch b_k}\br*{\frac{2\theta_{k}}{\dr{j,\min}} +\mu_{k}\br*{1 + \log \frac{\dr{j,\max}}{\dr{j,\min}}}} \\
    & \le \brs*{1728\gamma_g^2\bar\Nitems^2\sum_{j=1}^{\Narms}  \frac{1}{\dr{j,\min}}+68\gamma_\infty \bar\Nitems \sum_{j=1}^{\Narms}  \br*{1 + \log\frac{\dr{j,\max}}{\dr{j,\min}}}}\br*{\sum_{k=1}^{k_0} \frac{a_k}{b_k}} \ell \log T \enspace .
\end{aligned}
\end{equation*}

Similarly to appendix C of \citep{degenne2016combinatorial}, choosing $a_k=b_k=b^k$ for $b=0.2$ 
\begin{equation*}
\begin{aligned}
    \br*{\sum_{k=1}^{k_0} \frac{a_k}{b_k}} \ell = \br*{\sum_{k=1}^{k_0} \frac{a_k}{b_k}}\brs*{\sum_{k=1}^{k_0}\frac{b_{k-1}-b_k}{a_k} + \frac{b_0}{a_0}} 
    \le \frac{1}{b} \ceil*{\frac{\log \Nbatch}{\log 1/b}}^2 
    \le 5 \ceil*{\frac{\log \Nbatch}{1.61}}^2 \enspace ,
\end{aligned}
\end{equation*}

which concludes the second part of the results.
\end{proof}

\section{Proof of Corollary \ref{corollary:independentRegret}}
\label{append:independentRegret}
\independentRegret*

\begin{proof}

We start from Lemma \ref{lemma:regretDecomp}, but divide the first term of the regret into large gaps $\dr{A_t}\ge\dr{}$ and small gaps $\dr{A_t}\le\dr{}$. We then use Lemma \ref{lemma:sumOfSmallGaps} only for the large gaps, and bound the regret of the smaller gaps by the trivial bound $\dr{}T$:
\begin{equation*}
\begin{aligned}
R(T) 
& \le \E \brs*{\sum_{t=\Narms+1}^T \dr{A_t}\Ind{\dr{A_t}\le \bar\Nitems c_t(A_t)}} + \Narms \dr{\max}\br*{ 1 + \Nitems\frac{2\pi^2}{3}} \\
& \le \E \brs*{\sum_{t=\Narms+1}^T \dr{A_t}\Ind{\dr{A_t}\le \bar\Nitems c_t(A_t),\dr{A_t}\ge\dr{}}} + \Narms \dr{\max} \br*{1 + \Nitems\frac{2\pi^2}{3}} \\ 
& \quad +\E \brs*{\sum_{t=1}^T \dr{A_t}\Ind{\dr{A_t}<\dr{}}} \\
& \le \brs*{8640\gamma_g^2\bar\Nitems^2\!\!\sum_{j:\dr{j,\min}>\dr{}} \!  \frac{1}{\dr{j,\min}}+340\gamma_\infty \bar\Nitems \!\!\sum_{j:\dr{j,\min}>\dr{}} \! \br*{1 + \log\frac{\dr{j,\max}}{\dr{j,\min}}}}\ceil*{\frac{\log \Nbatch}{1.61}}^2 \log T \\
& \quad +\Narms\dr{\max} \br*{ 1 + \Nitems\frac{2\pi^2}{3}}  + T\dr{}\\
& \le \brs*{8640\gamma_g^2\bar\Nitems^2 \frac{\Narms}{\dr{}}+340\gamma_\infty \bar\Nitems \Narms \br*{1 + \log\frac{\dr{\max}}{\dr{}}}}\ceil*{\frac{\log \Nbatch}{1.61}}^2 \log T \\
& \quad + \Narms\dr{\max} \br*{ 1 + \Nitems\frac{2\pi^2}{3}}  + T\dr{} .
\end{aligned}
\end{equation*}
Next, set $\dr{} = \sqrt{\frac{u_1\log T}{T}} + \frac{u_2\log T}{T}$, for $u_1=8640\gamma_g^2\bar\Nitems^2\Narms\ceil*{\frac{\log \Nbatch}{1.61}}^2$ and $u_2=340\gamma_\infty \bar\Nitems \Narms\ceil*{\frac{\log \Nbatch}{1.61}}^2$. 
Substituting this value gives the desired results:

\begin{equation*}
\begin{aligned}
R(T) 
& \le \brs*{ \frac{8640\gamma_g^2\bar\Nitems^2\Narms}{\sqrt{\frac{u_1\log T}{T}} + \frac{u_2\log T}{T}}+340\gamma_\infty \bar\Nitems \Narms \log\frac{\dr{\max}}{\sqrt{\frac{u_1\log T}{T}} + \frac{u_2\log T}{T}}}\ceil*{\frac{\log \Nbatch}{1.61}}^2 \log T \\
& \quad + 340\gamma_\infty \bar\Nitems\Narms\ceil*{\frac{\log \Nbatch}{1.61}}^2 \log T  +  \Narms\dr{\max} \br*{ 1 + \Nitems\frac{2\pi^2}{3}}  + T\br*{\sqrt{\frac{u_1\log T}{T}} + \frac{u_2\log T}{T}} \\
& \le \brs*{ \frac{8640\gamma_g^2\bar\Nitems^2\Narms}{\sqrt{\frac{u_1\log T}{T}}}+340\gamma_\infty \bar\Nitems \Narms \log\frac{\dr{\max}}{\frac{u_2\log T}{T}}}\ceil*{\frac{\log \Nbatch}{1.61}}^2 \log T \\
& \quad + 340\gamma_\infty \bar\Nitems\Narms\ceil*{\frac{\log \Nbatch}{1.61}}^2 \log T  +  \Narms\dr{\max} \br*{ 1 + \Nitems\frac{2\pi^2}{3}}  + \sqrt{u_1 T\log T} + u_2\log T \\ 
& = 2\sqrt{8640}\gamma_g\bar\Nitems \ceil*{\frac{\log \Nbatch}{1.61}} \sqrt{\Narms T\log T} +  \Narms\dr{\max} \br*{ 1 + \Nitems\frac{2\pi^2}{3}}\\
& \quad + 340\gamma_\infty \bar\Nitems\Narms\ceil*{\frac{\log \Nbatch}{1.61}}^2 \log T
\br*{2 + \log\frac{\dr{\max}T}{340\gamma_\infty \bar\Nitems \Narms\ceil*{\frac{\log \Nbatch}{1.61}}^2\log T}}
\end{aligned}
\end{equation*}

\end{proof}

\end{document}